\documentclass[runningheads,a4paper]{llncs}

\usepackage{amsmath, amssymb, multirow, paralist}
\usepackage{latexsym}
\usepackage{amsmath}
\usepackage{amstext}
\usepackage{amsfonts}
\usepackage{amsopn}

\usepackage{hyperref}
\hypersetup{
colorlinks=true,
citecolor=blue,
linkbordercolor={1  0 1}, 
citebordercolor={1 1 1} 
}

\usepackage{paralist}
\usepackage{fancybox}
\newenvironment{algorithm}[1][\  ] %
{
\begin{Sbox}\begin{minipage}{4in}\vspace*{0.1cm}
 \rm
\begin{tabbing}
....\=...\=...\=...\=...\=  \+ \kill
}%
{\end{tabbing}
\vspace*{0.1cm}\end{minipage}\hspace*{0.4cm}
\end{Sbox}\fbox{\TheSbox}
}

\usepackage{url}
\urldef{\mailsa}\path|{mahdavim, yangtia1, rongjin}@msu.edu|    
\newcommand{\keywords}[1]{\par\addvspace\baselineskip
\noindent\keywordname\enspace\ignorespaces#1}

\def \E {\mathrm{E}}
\def \x {\mathbf{x}}
\def \w {\mathbf{w}}
\def \p {\mathbf{p}}
\def \q {\mathbf{q}}
\def \R {\mathbf{R}}
\def \r {\mathbf{r}}

\def \c {\mathbf{c}}

\usepackage{amsmath, amssymb, multirow, paralist}
\usepackage{latexsym}
\usepackage{amstext}
\usepackage{amsfonts}
\usepackage{amsopn}

\begin{document}

\mainmatter  

\title{Efficient  Constrained Regret Minimization}


\titlerunning{Constrained Regret Minimization}

%
%
\author{Mehrdad Mahdavi%
\and Tianbao Yang\and Rong Jin \\
}
%
\authorrunning{M. Mahdavi, T. Yang, and R. Jin}

\institute{Department of Computer Science and Engineering \\
Michigan State University, MI, 48824, USA\\
\mailsa
}
%
%



\toctitle{Lecture Notes in Computer Science}
\tocauthor{Authors' Instructions}
\maketitle

\begin{abstract}  Online learning constitutes a mathematical and compelling  framework to analyze sequential decision making problems in adversarial environments. The learner repeatedly chooses an action, the environment responds with an outcome, and then the learner receives a reward  for the played action. The goal of the learner is to maximize his total  reward.  However, there are situations in which, in addition to maximizing the cumulative reward, there are some additional constraints on the sequence of decisions that must be satisfied on average by the  learner.  
In this paper we study an extension to the online learning where the learner aims to maximize the total reward given that some additional constraints need to be satisfied.  By leveraging on the  theory of Lagrangian method in constrained optimization, we propose Lagrangian exponentially weighted average ({LEWA}) algorithm, which is a primal-dual variant of the well known exponentially weighted average algorithm,  to efficiently  solve constrained online decision making problems. Using novel theoretical analysis, we establish the regret and the violation of the constraint bounds in full information and bandit feedback models. 
\end{abstract}


\section{Introduction}

Many practical problems such as  online portfolio management~\cite{DBLP:conf/nips/HazanK09a}, prediction from expert advice~\cite{DBLP:journals/jacm/Cesa-BianchiFHHSW97,Littlestone:1994:WMA:184036.184040}, and online shortest path problem~\cite{Takimoto:2003:PKM:945365.964295}, involve making repeated decisions in an unknown and unpredictable environment (see, e.g.~\cite{Cesa-Bianchi:2006:PLG:1137817} for a comprehensive review). These situations  can be formulated as a repeated game between the decision maker (i.e., the learner) and the adversary (i.e., the environment). At each round of  the game, the learner selects an action from a fixed set of actions and then receives feedback (i.e., reward) for the selected action.  In the adversarial or non-stochastic feedback model, we make no statistical assumption on the sequence of rewards except that the rewards are bounded. The player would like to learn from the past and hopefully make better decisions as time goes by, so that the total accumulated reward is large.  

The analysis of online learning algorithms focuses on establishing bounds on the \textit{regret} that is the difference between the reward of the best fixed action with the hindsight knowledge of the observed sequence and the cumulative reward of the online learner.  If the online algorithm attains sublinear bound on the regret, is said to be Hannan consistent~\cite{Cesa-Bianchi:2006:PLG:1137817}, which indicates that in the long run, the learner's average reward per round approaches the average reward per round of the best action. A point worthy of notice is that  the performance bound must hold for any sequence of rewards, and in particular  if the sequence is chosen adversarially. We also note that this setting differs from the framework of \textit{competitive analysis} where the decision maker is allowed to first observe the reward vector, and  then make the decision and get the reward accordingly~\cite{competitive}.

In many current literature, the application of online learning is mostly limited to problems without  constraints  on the decisions.  However, in most scenarios, beyond maximizing the cumulative reward, there are some restrictions on the sequence of decisions made by the learner that need to be satisfied on the average. Moreover, in some applications it seems beneficial to sacrifice some reward to get along with other goals simultaneously. Therefore, one might desire algorithms for a much more ambitious framework, where we need to maximize total reward under the  constraints defined on the sequence of decisions. Attempts for such extension were made in~\cite{DBLP:journals/jmlr/MannorTY09}, where the online learning with path constraints  has been addressed and algorithms with \textit{asymptotically} vanishing bound have been proposed. 

As an illustrative example, let us consider  a wireless communication system where the agent chooses an appropriate transmission power in order to transmit a message successfully.  If one considers the amount of power required to transmit a packet through a path as its cost, the goal of the agent may be to maximize
average throughput, while keeping the average power consumption under some required threshold. As another motivating example, consider the online ads placement with budgeted advertisers. This problem can be cast as a multi armed bandit (MAB) problem, with the set of arms being the set of ads. Since each advertiser has a limited budget to represent his adds, the online learner must consider the budget restriction of each advertiser in making decisions.  

To model abovementioned situations, we consider modifying the online learning problem to achieve both goals simultaneously where the additional goal is called  constraint throughout the paper to distinguish it from the regret. Roughly speaking, we try to devise online algorithms in order to maximize the revenue and to some degree guarantee vanishing bound on the additional constraint.  The
constraint defined over the actions  necessitates a compromise: if the algorithm be too aggressive to satisfy the constraint, then there would be less hope to attain satisfactory cumulative reward at the end of the game and on the other hand, just trying to maximize the cumulative reward will end up in a situation in which the constraint vanishes linearly in terms of the number of rounds.

An algorithm  addressing this problem has to balance between maximizing the adversary rewards and satisfying the constraint. To affirmatively address the problem, we provide a general framework for repeated games with constraint, and  propose a simple randomized algorithm called \textbf{Lagrangian exponentially weighted average (LEWA)} algorithm  for a particular class of these games. The proposed  formulation is inspired by the theory of Lagrangian method in constrained optimization and is based on primal-dual formulation of the exponentially weighted average (EWA) algorithm~\cite{Littlestone:1994:WMA:184036.184040} \cite{Freund:1997:DGO:261540.261549}. To the best of our knowledge, this is the first time a Lagrangian style relaxation has been proposed for this type of problem. 

The contribution of the present work  is to 1) introduce a general primal-dual framework for solving online learning with constraints problem;  2) propose a Lagrangian based exponentially weighted average algorithm for solving repeated games with constraints; 3) establish expected and high probability bounds on the regret and the violation of the constraints on average; 4) extend the results to the bandit setting where only partial feedback about the rewards and constraints are available.  \\

\noindent \textbf{Notations.} Before proceeding, we define the notations used in this paper. Vectors are indicated in lower case bold letters such as $\x$ where $\x^{\top}$ denotes it transpose. By default, all vectors are column vectors. For a vector $\mathbf{x}$, $x_i$ denotes its $i$th coordinate.  We use superscripts to index rounds of the game. Component-wise multiplication between vectors is denoted by $\circ$. We use $[K]$ as a shorthand for the set of integers $\{1, 2, \ldots, K\}$. Throughout the paper we denote by  $[\cdot]_+$ the projection onto the positive orthant.  We shall use $\boldsymbol{1}$ to denote the vector of all ones.  Finally, for a $K$-dimensional vector $\x$,  $(\x)^2$ represents $(x_1^2, \ldots, x_K^2)$.
\section{Statement of the Problem}
We consider the general decision-theoretic framework for online learning and extend it to capture the constraint. In original online decision making, the learner is given access to a pool of $K$ actions. In each round $t \in [T]$, the learner chooses a probability distribution
$\p_t = (p^t_1, . . . , p^t_K)$ over the actions $[K]$ and chooses an action $i$ randomly based on $\p_t$. 
In the scenario of \textit{full information},  at each iteration, the adversary reveals a reward vector $\r_t=(r^t_1,\cdots, r^t_K)$. Choosing  an action $i$ results in receiving a reward $r_i^t$, which we shall assume without loss of generality to be bounded in $[0,1]$. In the \textit{partial information} or \textit{bandit} setting, only the cost of selected action is revealed by the adversary.  The learner competes with the best fixed action in hindsight and his/her goal is to minimize the regret defined as 
\begin{align*}
 \text{Regret}_T= \max_{\p}\sum_{t=1}^{T}{ \p^{\top}\r_t} -  \sum_{t=1}^{T}\p_t^{\top}\r_t . 
 \end{align*}

This problem is a well studied problem and there are algorithms which attain an optimal regret bound of $O(\sqrt{ T \ln K})$ after $T$ rounds of the game. In this paper we focus on exponentially weighted average (EWA), which will be used later as the baseline of the proposed algorithm.  The EWA algorithm maintains a weight vector $\w_t = (w^t_1, \cdots, w^t_K)$ which is used to define the probabilities over actions. After receiving the reward vector $\r_t$ at round $t$, the EWA algorithm updates the weight vector according to $w^{t+1}_i = w^{t}_i \exp(\eta r^t_i)$ where  $\eta$ is learning rate. 

In the new setting  addressed in this paper, which we refer to as  \textit{constrained regret minimization}, in addition to the rewards, there exist some constraints on the decisions that need to be satisfied. In particular, for the decision $\p$ made by the learner, there is an additional constraint $\p^{\top} \c\geq c_0$ where $\c$ is a constraint vector for specifying the constraint  (e.g. the cost vector for the arms in MAB problem).  We note that, in general, the reward vector $\r_t$ and the constraint vector $\c$ are different and can not be combined as a single objective. The learner's goal is to maximize the total reward with respect to the optimal decision  in hindsight under the  constraint $\p^{\top}\c\geq c_0$, i.e., 
\begin{align*}
\min_{\p_1, \ldots, \p_T }\;\;\max_{\p^{\top}{\c}\geq c_0}\sum_{t=1}^{T}{ \p^{\top}\r_t}-\sum_{t=1}^{T}{\p_t^{\top}\r_t} , 
\end{align*}
and simultaneously satisfy the constraint. Note that  the comparator class includes  fixed decision $\p$ that attains maximal cumulative reward  had he known the rewards beforehand, while satisfying the  additional  constraint.

 Within our setting, we consider repeated games with \textit{adversarial} rewards and \textit{stochastic} constraint. More precisely, let $\c=(c_1,\cdots, c_K)$ be the constraint vector defined over actions. In stochastic setting the vector $\c$ is \textbf{unknown} to the learner and in each round $t \in [T]$, beyond the reward feedback, the learner receives a random realization $\c_t=(c^t_1,\cdots, c^t_K)$  of $\c$ where $\E[c^t_i]=c_i$.  The learner's goal is to  choose a sequence of decisions  $\p_t, t \in[T]$ to minimize the regret with respect to the optimal decision  in hindsight under the  constraint $\p^{\top}\c\geq c_0$. Without loss of generality we assume $\c_t\in [0, 1]^K $ and  $c_0\in[0, 1]$. Formally,  the goal of the learner is to attain a gradually vanishing \textit{constrained regret}  as
\begin{align}
\text{Regret}_T=\max_{\p^{\top}\c\geq c_0} \sum_t \p^{\top}\r_t -  \sum_{t}\p_t^{\top}\r_t\leq O(T^{1-\beta_1}).
\end{align}
Furthermore, the decisions $\p_t, t=1,\cdots, T$ made by the learner are required to attain sub-linear bound on  the violation of the constraint in long run, i.e., 
\begin{align}
\text{Violation}_T=\left[\sum_{t=1}^T\left(c_0- \p_t^{\top}\c\right)\right]_+\leq O(T^{1-\beta_2}).
\end{align}
We refer to the above bound as the violation of the constraint.  We distinguish two different types of constraint satisfaction algorithms: \textit{one shot} and \textit{long term} satisfaction. In one shot constraint satisfaction, the learner is required to satisfy the constraint at each round, i.e., $\p_t^{\top}\c\geq c_0$. In contrast, in the long term version, the learner is allowed to violate the constraint for some rounds in a controlled way; but the constraint must hold on average for all rounds, i.e., $(\sum_{t=1}^{T}{\p_t^{\top} \c) /T\geq c_0}$.

The main questions addressed in this paper are how to modify EWA algorithm to take the constraints under consideration and what would be the  bounds on the regret as well as the violation of the constraints attainable by the modified algorithm.

\section{Related Works}


As is well known, a wide range of literature deals with the online decision making problem without constraints and there exist a number of regret-minimizing algorithms that have the optimal regret bound. The most well-known and successful work is probably the Hedge algorithm~\cite{Freund:1997:DGO:261540.261549}, which was a direct generalization of Littlestone and Warmuth's Weighted Majority (WM) algorithm \cite{Littlestone:1994:WMA:184036.184040}. Other
recent studies include the improved theoretical bounds and the
parameter-free hedging algorithm \cite{DBLP:conf/nips/ChaudhuriFH09} and adaptive Hedge \cite{DBLP:conf/nips/ErvenGKR11} for decision-theoretic online learning. We refer readers to the \cite{Cesa-Bianchi:2006:PLG:1137817} for an  in-depth discussion of this subject.


As the first seminal paper in adversarial setting, Mannor et al. \cite{DBLP:journals/jmlr/MannorTY09} introduced the online learning with simple path constraints. They considered the  infinitely repeated two player games with stochastic rewards  where for every joint action of the players, there is an additional stochastic constraint vector that is accumulated by the decision maker. The learner is asked to keep the cumulative constraint vector in a predefined set in the space of constraint vectors. They showed that if the convex set is affected by both decisions and rewards, the  optimal reward is generally unattainable online. The positive result is that a relaxed goal, which is defined in terms of the convex hull of the constrained reward in hindsight is attainable. For the relaxed setting, they suggested two inefficient algorithms: one relies  on Blackwell's approachability theory and the other is based on calibrated forecast of the adversary's actions. Given the implementation difficulties associated with these two methods, they suggested two efficient heuristic methods to attain the reward with meeting the constraint in the long run.  We note that the analysis in \cite{DBLP:journals/jmlr/MannorTY09} is asymptotic while the bounds to be established in this work are  applicable to finite repeated games. 

In \cite{DBLP:conf/aaai/LongCCRJ10} the budget limited MAB was introduced where polling an arm is costly  where the cost of each arm is fixed in advance. In this setting both the exploration and exploitation phases are limited by a global budget. This setting matches the stochastic rewards with deterministic constraints without violation game discussed before.  It has been shown that existing MAB algorithms are not suitable to efficiently deal with costly arms.  They proposed the $\epsilon-first$ algorithm that dedicates the first $\epsilon$ fraction of the  total budget exclusively  for exploration and the remaining $(1-\epsilon)$ fraction for exploitation. \cite{DBLP:journals/corr/abs-1204-1909} improves the bound obtained in \cite{DBLP:conf/aaai/LongCCRJ10} by proposing a knapsack based UCB \cite{DBLP:journals/corr/abs-1204-1909} algorithm which extends the UCB algorithm by solving a knapsack problem at each round to cope with the constraints. We note that knapsack based UCB does not make explicit distinction between exploration and exploitation steps as done in $\epsilon-first$ algorithm. In both \cite{DBLP:journals/corr/abs-1204-1909} and \cite{DBLP:conf/aaai/LongCCRJ10} the algorithm proceeds as long as sufficient budget existing to play the arms. 

Finally, we remark that our setting differs from the setting considered in~\cite{DBLP:conf/nips/AbernethyW10} which puts restrictions on the actions taken by the adversary and not the learner as in our case.

\section{Full Information Constrained Regret Minimization}
In this section, we present  the basic algorithm for the online learning with constraint  problem and analyze its performance via the primal-dual method in adversarial setting.

A straightforward approach to tackle the problem is to modify the  reward functions of the learner  to include
constraint term with a  penalty coefficient that adjust the probability of the actions  when the  constraint is violated. This approach circumvents the problem of a constrained online learning by turning it into an unconstrained
problem. But a simple analysis shows that, in the adversarial setting, this simple penalty based approach fails to attain gradually vanishing bounds for regret and the violation of constraint. The main difficulty arises from the fact that an adaptive adversary can play  with the penalty coefficient associated with the constraint in order to weaken the influence of the penalty parameter  which results in linear bound on at least one of the measures, i.e. either regret bound or violation of the constraint.

Alternatively, since the constraint vector in our setting is stochastic, one possible solution is to take an exploration and exploitation scheme, i.e., to burn a small portion $\epsilon$  of the rounds  to estimate the constraint vector $\c$ by $\widetilde{\c}$ and then in the remaining $(1-\epsilon)T$ rounds follow the existing algorithms with restricted decisions, i.e., $\p \in \Delta_K \cap \p^{\top} \widetilde{\c}\geq c_0$, where $\Delta_K$ is the simplex over $[K]$.  The parameter $\epsilon$ balances  the accuracy of estimating $\c$ and the number of rounds for exploitation to increase the total reward. One may hope that by careful adjustment of  $\epsilon$, it would be possible to get satisfactory bounds on regret and the violation of the constraint. But unfortunately this naive approach suffers from two main drawbacks. First, the number of rounds $T$ is not known in advance. Second, the decisions are made by projecting into an estimated  domain $\p^{\top}\widetilde\c\geq c_0$ instead  of the true domain $\p^{\top}\c\geq c_0$ which is problematic as follows. In order to show the regret bound, we need to relate the best cumulative reward in the estimated domain to that in the true domain, which however requires imposing a regularity condition on reward and constrain vectors to be solvable~\cite{STAB-LP}.  Basically, we can make the algorithm adaptive to $T$ by using a similar idea to \textit{epoch greedy} \cite{DBLP:conf/nips/LangfordZ07} algorithm that runs exploration/exploitation in epochs, but it still suffers from the second drawback. Additionally, projection to the inaccurate  estimated constraint $\widetilde{\c}$ does not exclude the possibility that the solution will be infeasible.

Here we take a different path to solve the problem. The proposed  algorithm is inspired by the theory of Lagrangian method in constrained optimization. The intuition behind the proposed algorithm is to optimize one criterion (i.e., minimizing regret or maximizing the reward) subject to explicit constraint on the restrictions that the learner needs to satisfy in average for the sequence of the decisions. A challenging ingredient in this formulation is that of establishing bounds on the regret and the violation of the constraint. In particular, our algorithms will exhibit a bound in the following structure, 
\begin{align}\label{eqn:str}
\text{Regret}_T + \frac{\text{Violation}_T^2}{O(T^{1-\alpha})}\leq O(T^{1-\beta}), 
\end{align}
where $\text{Violation}_T$ is a term related to the violation of the constraint in long term. From (\ref{eqn:str}) we can derive a bound on regret and the violation of the constraint as 
\begin{align}
\text{Regret}_T&\leq O(T^{1-\beta})\label{eqn:reg} \\
\text{Violation}_T&\leq \sqrt{O\left([T+T^{1-\beta}]T^{1-\alpha}\right)},\label{eqn:vc}
\end{align}
where the last  bound follows  the fact $-\text{Regret}_T\leq O(T)$.

\begin{figure}[t]
  \begin{center}
\begin{algorithm}
{\bf \large LEWA ($\eta$ {\textnormal {and}} $\delta$)} \\
initialize: \+  $\w_1 = \boldsymbol{1}$ and $\lambda_1 = 0$ \- \\
{\bf iterate} $t=1,2,\ldots, T$  \+ \\
Draw an action accordingly to the probability  $\displaystyle \p_t={\w_t}/\sum_jw^t_j$ \\
Receive reward $\r_t$ and  a realization of constraint $\c_t$ \\
Update  $\displaystyle \w_{t+1}= \w_t\circ \exp(\eta (\r_t +  \lambda_{t}\c_t))$ \\
Update $\lambda_{t+1} =  [(1-\delta\eta)\lambda_{t} - \eta (\p_{t}^{\top}\c_{t}-c_0)]_+$ \- \\ 
{\bf end iterate} 
\end{algorithm}
\end{center}
\label{lewa:1}
\caption{Lagrangian exponentially weighted average for full information online decision making under constraints}
\end{figure}

The detailed steps of the proposed algorithm  are shown in LEWA. The algorithm keeps two set of variables: the weight vector $\w_t$ and the Lagrangian multiplier $\lambda_t$. The  high level interpretation of the algorithm is as follows: if the constraint is being violated a lot, the decision maker places more weight on the constraint  controlled by $\lambda_t$; but it tunes down the weight on the  constraint when the constraint is satisfied reasonably.  We note the LEWA is equivalent to the original EWA when the constraint is satisfied at each iteration, i.e., $\p_t^{\top}\c_t\geq c_0$, which gives $\lambda_1 = \cdots =  \lambda_t =\ldots = 0$.
 It should be emphasized that in some previous works such as \cite{DBLP:conf/aaai/LongCCRJ10},  the 
 learner is not allowed to exceed the pre-specified threshold for the violation of the constraint and the game stops as soon as the learner violates the constraint.  In contrast, within our setting, the learner's goal  is to obtain sub-linear bound on the long term violation of the constraint. 
 
 We now state the main theorem about the performance of  LEWA algorithm.  
\begin{theorem}
\label{thm:full_exp}
Let $\p_1, \p_2, \cdots, \p_T$ be the sequence of randomized decisions  over the set of actions $[K]:= \{1,2,\cdots, K \}$ produced by LEWA algorithm under the sequence of adversarial rewards $\r_1, \r_2, \cdots, \r_T \in [0,1]^K$ observed for these decisions. Let $\lambda_1, \lambda_2, \cdots, \lambda_T$ be the corresponding dual sequence. By setting $\displaystyle \eta=\sqrt{4\ln K / (9T)}$ and $\delta=\eta/2$ we have:
\begin{align*}
\max_{\p^{\top}\c\geq c_0}\sum_{t=1}^T \p^{\top} \r_t  -\E\left[\sum_{t=1}^T \p_t^{\top} \r_t \right]&\leq 3\sqrt{T\ln K} \; \text{and}\\  \; \E\left[\sum_{t=1}^T(c_0-\p_t^{\top}\c)\right]_+\leq O(T^{3/4}),
\end{align*}
where expectation is taken over randomness in $\c_1,\cdots, \c_T$.
\end{theorem}
From Theorem \ref{thm:full_exp} we see that the LEWA algorithm attains the optimal bound for the regret and an $O(T^{3/4})$ bound on the violation of the constraint. Before proving the Theorem \ref{thm:full_exp}, we state two lemmas that pave the way to the proof of theorem. 
\begin{lemma}\label{lemma:primal}\textnormal{[Primal Inequality]}  Let $\R_t =\R^1_t+ \lambda_t \R^2_t$, where $\R^1_t, \R^2_t\in \mathbb R_+^K$, $\w_{t+1}= \w_t\circ\exp(\eta\R_t)$, and $\p_t=\w_t/\w_t^{\top}\mathbf 1$. 
Assuming $\max(\|\R^1_t\|_\infty,\|\R^2_t\|_\infty)\leq s$, we have the following primal equality 
\begin{align}
\label{eqn:primal1}
\sum_{t=1}^T(\p - \p_t)^{\top}\R_t \leq \frac{\ln K}{\eta} + s^2\left(\frac{\eta T}{4}+ \frac{\eta}{4}\sum_{t=1}^{T}{\lambda_t^2}\right).
\end{align}
\end{lemma}
\begin{proof}
Let $W_t=\sum_{i=1}^Kw^t_i$. We first show an upper bound and a lower bound on $\ln W_{T+1}/W_{1}$, followed by combining the bounds together. We have
\begin{align*}
&\sum_{t=1}^T \ln \frac{W_{t+1}}{W_t} =\ln\frac{W_{T+1}}{W_1} \\ &= \ln \sum_{i=1}^Kw^{T+1}_i - \ln K\geq \ln\sum_{i=1}^K p_iw_i^{T+1} - \ln K \geq \eta \p^{\top}\sum_{t=1}^T\R_t - \ln K,
\end{align*}
where the last inequality follows from the concavity of the $\log$ function. By following Lemma 2.2 in~\cite{Cesa-Bianchi:2006:PLG:1137817}, we obtain
\begin{align*}
& \sum_{t=1}^T \ln\frac{W_{t+1}}{W_t}= \sum_{t=1}^{T}\sum_{i=1}^{K}{\frac{w_i^t\exp(\eta R_i^t)}{\sum_{j=1}^Kw_j^t}} \\ & \leq \eta \sum_{t=1}^{T}{\sum_{i=1}^K\frac{w^t_i}{\sum_{j=1}^K w_j^t}R_i^t + \frac{\eta^2}{8}s^2(1+\lambda_t)^2} \leq \eta \sum_{t=1}^{T}{\p_t^{\top}\R_t + \frac{\eta^2}{8}\sum_{t=1}^{T}s^2{(1+\lambda_t)^2}}.
\end{align*}
Combining the lower and upper bounds and using the inequality $(a+b)^2 \leq 2(a^2+b^2)$, we obtain the desired inequality in (\ref{eqn:primal1}).
\end{proof}

\begin{lemma}\label{lemma:dual}\textnormal{[Dual Inequality]}
Let $g_t(\lambda)= \frac{\delta}{2}\lambda^2 + \lambda (\beta_t - c_0)$, $\lambda_{t+1}=[(\lambda_t - \eta\nabla g_t(\lambda_t)]_+$, and $\lambda_1=0$. Assuming $\eta>0, 0\leq \beta_t\leq \beta_0$, we have
\begin{align}\label{eqn:dual1}
\sum_{t=1}^{T}{(\lambda_t-\lambda)(\beta_t- c_0)}+\frac{\delta}{2} \sum_{t=1}^{T}{(\lambda_t^2-\lambda^2)} \leq \frac{\lambda^2}{2 \eta} + (c_0^2+\beta_0^2) \eta T.
\end{align}
\end{lemma}
\begin{proof}
First we note that
 \begin{align*}
\lambda_{t+1} &= [\lambda_t - \eta\nabla g_t(\lambda_t)]_+ \\&=[(1-\delta\eta)\lambda_t - \eta(\beta_t-c_0)]_+\leq [(1-\delta\eta)\lambda_t  + \eta c_0]_+.
\end{align*}
By induction on $\lambda_t$, we can obtain $\displaystyle \lambda_t\leq\frac{c_0}{\delta}$.  
Applying the standard analysis of online gradient descent~\cite{DBLP:conf/icml/Zinkevich03} yields 
\begin{align*}
|\lambda_{t+1} - \lambda|^2& = |{\Pi}_+[\lambda_t - \eta( \delta\lambda_t + \beta_t  -c_0)]-\lambda|^2\\
&\leq |\lambda_t - \lambda|^2  + | \eta (\delta\lambda_t-c_0) + \eta \beta_t|^2  - 2(\lambda_t -\lambda)(\eta\nabla g_t(\lambda_t))\\
&\leq |\lambda_t - \lambda|^2 + 2\eta^2c_0^2 + 2\eta^2\beta_0^2 + 2\eta(g_t(\lambda) - g_t(\lambda_t)).
\end{align*}
Then, by rearranging the terms we get 
\begin{align*}
g_t(\lambda_t) - g_t(\lambda)\leq  \frac{1}{2\eta}\left(|\lambda_{t+1}- \lambda|^2 - |\lambda_t - \lambda|^2\right) + \eta(c_0^2+\beta_0^2).
\end{align*}
Expanding the terms on l.h.s and taking the sum over $t$, we obtain the inequality as desired.
\end{proof}
\begin{proof}\textnormal{[of Theorem \ref{thm:full_exp}]} Applying $\R_t = \r_t + \lambda_t\c_t$ to the primal inequality in Lemma~\ref{lemma:primal}, where $\max(\|\r_t\|_\infty, \|\c_t\|_\infty)\leq 1$,  we have
\begin{align*}
\sum_{t=1}^T(\p - \p_t)^{\top}(\r_t+\lambda_t\c_t) \leq \frac{\ln K}{\eta} + \frac{\eta T}{4}+ \frac{\eta}{4}\sum_{t=1}^{T}{\lambda_t^2}.
\end{align*}
Applying $\beta_t=\p_t^{\top}\c_t$ to the dual inequality in Lemma~\ref{lemma:dual}, where $\beta_t\leq 1, c_0\leq 1$, we have
\begin{align*}
\sum_{t=1}^{T}{(\lambda_t-\lambda)(\p_t^{\top}\c_t- c_0)}+\frac{\delta}{2} \sum_{t=1}^{T}{(\lambda_t^2-\lambda^2)} \leq \frac{\lambda^2}{2 \eta} + 2\eta T.
\end{align*}
Combining the above two inequalities gives
\begin{align*}
&\sum_{t=1}^T (\p^{\top}\r_t - \p_t^{\top}\r_t) + \sum_{t=1}^T\lambda(c_0-\p_t^{\top}\c_t)- \left(\frac{\delta T}{2}+\frac{1}{2\eta}\right)\lambda^2\\
& \leq  \frac{\ln K}{\eta} +\frac{9\eta T}{4} + \left(\frac{\eta}{4}-\frac{\delta}{2}\right)\sum_{t=1}^T\lambda_t^2 + \sum_{t=1}\lambda_t(c_0-\p^{\top}\c_t).
\end{align*}
Taking expectation over $\c_t, t=1,\cdots, T$, by using $\E[\c_t]=\c$ and noting that $\p_t$ and $\lambda_t$ are independent of $\c_t$, we have
\begin{align*}
&\E\left[\sum_{t=1}^T\left(\p^{\top} \r_t  - \p_t^{\top} \r_t\right) + \sum_{t=1}^T\lambda (c_0-\p_t^{\top}\c) - \left(\frac{\delta T}{2}+\frac{1}{2\eta}\right)\lambda^2\right]\\
&\leq  \frac{\ln K}{\eta}+\frac{9}{4}\eta T +\E\left[\left(\frac{\eta}{4} -\frac{\delta}{2}\right) \sum_{t=1}^T\lambda_t^2\right] + \E\left[\sum_{t=1}^T \lambda_t(c_0-\p^{\top}\c)\right].
\end{align*}
Let $\p$ be the solution satisfying $\p^{\top}\c\geq c_0$.  Noting that $\frac{\eta}{4} -\frac{\delta}{2}\leq 0$ and taking maximization over $\lambda>0$ in  l.h.s, we get 
\begin{align*}
\E\left[\max_{\p^{\top}\c\geq c_0}\sum_{t=1}^T \p^{\top} \r_t  - \p_t^{\top} \r_t \right]+\E\left[\frac{\left[\sum_{t=1}^T(c_0-\p_t^{\top}\c)\right]_+^2}{2(\delta T+1/\eta)}\right]\leq  \frac{\ln K}{\eta}+\frac{9}{4}\eta T.
\end{align*}
By plugging the values of $\eta$ and $\delta$, and noting the similar structure of above inequality as in~(\ref{eqn:str}) and writing in~(\ref{eqn:reg}) and~(\ref{eqn:vc}) formats, we obtain  the desired bound for regret and the violation of the constraints in long term. 
\end{proof}

\begin{remark} We note that when deriving the bound for $\text{Violation}_T$, we simply use a weak lower bound on regret as $\text{Regret}_T\geq -T$. It is possible to obtain an improved bound by considering tighter bound for the $\text{Regret}_T$. One way to do this is to bound the regret by the variation of the reward vectors  as
$\text{Variation}_T = \sum_{t=1}^T \|\r_t-\widehat \r_T\|_{\infty}$, where $\widehat\r_T=(1/T)\sum_{t=1}^T\r_t$ denotes the mean of $\r_t, t \in [T]$.  The analysis in~\ref{variation-proof}  bounds the violation of the constraint  in terms of $\text{Variation}_T$ as
\begin{align*}
 \left[\sum_{t=1}^T(c_0- \x_t^{\top}\c)\right]_+\leq  O(\sqrt{T})+ O(T^{1/4}\sqrt{\text{Variation}_T}).
 \end{align*}
 This bound is significantly better  when the variation of the reward vectors is small and in worst case it attains an $O(T^{3/4})$ bound similar to  Theorem \ref{thm:full_exp}. 
\end{remark}
\subsection{A High Probability Bound}
\begin{figure}[t]
  \begin{center}
\begin{algorithm}
{\bf \large High Probability LEWA ($\eta$, $\delta$ {\textnormal {and}} $\epsilon$)} \\
initialize: \+  $\w_1 = \boldsymbol{1}$ and $\lambda_1 = 0$ \- \\
{\bf iterate} $t=1,2,\ldots, T$  \+ \\
Draw an action accordingly to the probability  $\p_t=\displaystyle \w_t/\sum_jw^t_j$. \\
Receive reward $\r_t$ and  a realization of constraint $\c_t$ \\
Compute average constraint estimate  $\overline\c_t=\displaystyle \frac{1}{t}\sum_{s=1}^t\c_s$ \\
Update $\displaystyle \w_{t+1} = \w_{t}\circ\exp(\eta (\r_{t} +  \lambda_{t}\overline \c_{t}))$ \\
Update $\lambda_{t+1} =  [(1-\delta\eta)\lambda_{t} - \eta (\p_{t}^{\top}\overline \c_{t}+\alpha_t-c_0)]_+$.  \- \\
{\bf end iterate}
\end{algorithm}
\end{center}
\caption{High Probability LEWA }
\label{alg:lewa2}
\end{figure}
The performance bounds proved in the previous section for the regret and the violation of the constraint only holds in expectation which may have  enormous fluctuations around its mean.  Here, with a simple trick, we present a modified version of the LEWA algorithm which attains similar bounds with  overwhelming probability. 
To this end, we slightly  change the original LEWA algorithm. More specifically, instead of using $\c_t$ in updating $\lambda_{t+1}$,  we use the average estimate and add a confidence bound to achieve a more accurate estimation of the constraint vector $\c$.  The following theorem bounds the regret and the violation  of the constrain in high probability  for the modified algorithm.
\begin{theorem}
Let $\alpha_t= \frac{1}{\sqrt{t}}\sqrt{(1/2)\ln\left(2/ \epsilon\right)}$,  $\displaystyle \eta=O(T^{-1/2})$, and $\delta=\eta/2$. By running Algorithm 2  we have with probability $1-\epsilon$
\begin{align*}
&\max_{\p^{\top}\c\geq c_0}\sum_{t=1}^T \p^{\top} \r_t  - \sum_{t=1}^T \p_t^{\top} \r_t \leq \widetilde O(T^{1/2})
\;  \text{and} \;\\
&\left[\sum_{t=1}^T(c_0-\p_t^{\top}\c)\right]_+\leq O(T^{3/4}),
\end{align*}
where $\widetilde O(\cdot)$ omits the $\log$ term in $T$.
\end{theorem}
\begin{proof} 
Applying $\R_t = \r_t + \lambda_t \overline\c_t$ to the primal inequality in Lemma~\ref{lemma:primal}, where $\max(\|\r_t\|_\infty, \|\overline\c_t\|_\infty)\leq 1$, we have
\begin{align*}
\sum_{t=1}^T(\p - \p_t)^{\top}(\r_t+\lambda_t\overline\c_t) \leq \frac{\ln K}{\eta} + \frac{\eta T}{4}+ \frac{\eta}{4}\sum_{t=1}^{T}{\lambda_t^2}.
\end{align*}
Applying $\beta_t=\p_t^{\top}\overline\c_t+\alpha_t$ to the dual inequality in Lemma~\ref{lemma:dual}, where $\beta_t\leq 1+\alpha_1$, and $c_0\leq 1$, we have
\begin{align*}
\sum_{t=1}^{T}{(\lambda_t-\lambda)(\p_t^{\top}\overline\c_t + \alpha_t- c_0)}+\frac{\delta}{2} \sum_{t=1}^{T}{(\lambda_t^2-\lambda^2)} \leq \frac{\lambda^2}{2 \eta} + [1+(1+\alpha_1)^2]\eta T.
\end{align*}
Combining the above two inequalities results in
\begin{align*}
&\sum_{t=1}^T\left(\p^{\top} \r_t  - \p_t^{\top} \r_t\right) + \sum_{t=1}^T\lambda (c_0-\p_t^{\top}\overline\c_t-\alpha_t) - \left(\frac{\delta T}{2}+\frac{1}{2\eta}\right)\lambda^2\\
&\leq  \frac{\ln K}{\eta}+\left(\frac{13}{4}+2\alpha^2_1\right)\eta T +\left[\left(\frac{\eta}{4} -\frac{\delta}{2}\right) \sum_{t=1}^T\lambda_t^2\right] + \left[\sum_{t=1}^T \lambda_t(c_0-\p^{\top}\overline\c_t-\alpha_t)\right].
\end{align*}
Let $\p$ be the solution satisfying $\p^{\top}\c\geq c_0$. Noting that $\frac{\eta}{4} -\frac{\delta}{2}\leq 0$, and with a probability $1-\epsilon$, 
\[
 |\p^{\top}\c- \p^{\top}\overline \c_t| \leq \alpha_t, 
 \]
which is due to the Hoeffding's inequality~\cite{DBLP:conf/ac/BoucheronLB03},  by  taking maximization over $\lambda>0$ on the l.h.s,  we have with a probability $1-\epsilon  T$, 
 \begin{align*}
\max_{\p^{\top}\c\geq c_0}\sum_{t=1}^T \p^{\top} \r_t  - \p_t^{\top} \r_t +\left[\frac{\left[\sum_{t=1}^T(c_0-\p_t^{\top}\overline\c_t-\alpha_t)\right]_+^2}{2(\delta T+1/\eta)}\right]\leq  \frac{\ln K}{\eta}+\left(\frac{13}{4}+2\alpha^2_1\right)\eta T
\end{align*}
Pluging the stated values of $\eta$ and $\delta$, we have, with a probability $1-\epsilon T$, 
\begin{align*}
\max_{\p^{\top}\c\geq c_0}\sum_{t=1}^T \p^{\top} \r_t  - \p_t^{\top} \r_t &\leq O\left(T^{1/2}\ln(1/\epsilon)\right)\\\left[\sum_{t=1}^T(c_0-\p_t^{\top}\c)\right]_+&\leq \sqrt{(T+T^{1/2}\ln(1/\epsilon))T^{1/2}} + \sum_{t}(\p_t^{\top}\overline\c_t+\alpha_t -\p_t^{\top}\c)\\
&\leq O(T^{3/4}) + 2\sum_{t=1}^T\alpha_t\leq O(T^{3/4}) + O\left(T^{1/2}\ln(1/\epsilon)\right).
\end{align*}
By replacing $\epsilon$ with $\epsilon/T$ and noting that $O(T^{1/2}\ln T)\leq O(T^{3/4})$, we obtain the results stated in the theorem. 
\end{proof}
\section{Bandit Constrained Regret Minimization}
In this section, we generalize our results to the bandit setting for both rewards and constraints. In the bandit setting, at each iteration, we are required to choose an action $i_t$ from the pool of the actions $[K]$. Then only the reward and the constraint feedback for action $i_t$ are revealed to the learner, i.e. $r^t_{i_t}, c^t_{i_t}$. In this case, we are interested in the  regret bound as $\max_{\p^{\top}\c\geq c_0}\sum_{t=1}^T\p^{\top}\r_t - \sum_{t=1}^T r^t_{i_t}$.
In the classical setting, i.e., without constraint, this problem can be solved in stochastic and adversarial settings by UCB and Exp3 algorithms proposed in~\cite{DBLP:journals/ml/AuerCF02} and~\cite{DBLP:journals/siamcomp/AuerCFS02}, respectively. 
The algorithm is shown in BanditLEWA algorithm which uses the similar idea to Exp3 for exploration and exploitation. 

Before presenting the performance bounds of the algorithm, let us introduce two vectors:  $\widehat \r_t$ is all zero vector except in $i_t$th component which is set to be $\widehat r^t_{i_t} = r^t_{i_t}/p^t_{i_t}$ and  similarly $\widehat \c_t$ is all zero vector except in $i_t$th component which is set to be $\widehat c^t_{i_t} = c^t_{i_t}/p^t_{i_t}$. It is easy to verify that  $\E_{i_t}[\widehat \r_t]=\r_t$ and $\E_{i_t}[\widehat \c_t] = \c_t$.  The following theorem shows that  BanditLEWA algorithm  achieves $O(T^{3/4})$ regret bound and $O(T^{3/4})$ bound on the violation of the constraint in expectation. 
\begin{theorem}
Let $\gamma = O(T^{-1/2}), \displaystyle \eta=\frac{\gamma}{K}\frac{\delta}{\delta+ 1}$,  by running BanditLEWA algorithm, we have 
\begin{align*}
&\max_{\p^{\top}\c\geq c_0}\sum_{t=1}^T \p^{\top} \r_t  - \E\left[\sum_{t=1}^Tr^t_{i_t} \right]\leq O(T^{3/4})\;\;\text{and}\;\; \\
& \E\left[\sum_{t=1}^T(c_0-\p_t^{\top}\c)\right]_+\leq O(T^{3/4}).
\end{align*}
\end{theorem}

\begin{figure}[t]
  \begin{center}
\begin{algorithm}
{\bf \large BanditLEWA ($\eta$, $\gamma$, {\textnormal {and}} $\delta$)} \\
initialize: \+  $\w_1 = \boldsymbol{1}$ and $\lambda_1 = 0$ \- \\
{\bf iterate} $t=1,2,\ldots, T$  \+ \\
Set $\displaystyle\mathbf q_t= \w_t /\sum_j w^t_j$ \\
Draw action $i_t$ randomly accordingly to $\p_t=(1-\gamma)\displaystyle \mathbf q_t + \gamma\frac{\mathbf 1}{K}$ \\
Receive reward $r^t_{i_t}$ and a realization of constraint $c^t_{i_t}$ for action $i_t$ \\
Update $\displaystyle w^{t+1}_i = w^{t}_i\exp(\eta (\widehat r^{t}_i +  \lambda_{t}\widehat c^{t}_i))$ \\
Update $\lambda_{t+1} = [(1-\gamma\eta)\lambda_{t} - \eta(\mathbf q_t^{\top}\widehat\c_t -c_0)]_+$ \- \\
{\bf end iterate}
\end{algorithm}
\end{center}
\caption{Constrained regret minimization with partial (bandit) feedback  about reward and constraint vectors }
\label{alg:lewa3}
\end{figure}

\begin{proof}
In order to have an improved analysis, we first derive an improved primal inequality and an improved dual inequality. Let $\R_t= \widehat \r_t + \lambda_t\widehat\c_t$. By following the analysis for Exp3 algorithm~\cite{DBLP:journals/siamcomp/AuerCFS02}, we have
\begin{align}\label{eqn:imprimal}
\sum_{t=1}^T \eta \q_t^{\top}\R_t + \eta^2\q_t^{\top}\R_t^2 \geq \ln\frac{W_{T+1}}{W_1} \geq \eta\p^{\top}\sum_{t=1}^{T}{\R_t} -\ln K.
\end{align}
Dividing both sides by $\eta$, and taking expectation we get
\begin{align}\label{eqn:primal2}
&\E\left[\p^{\top}\sum_{t=1}^T \R_t - \sum_{t=1}^T  \q_t^{\top}\R_t\right]\leq \frac{\ln K}{\eta}+ \eta\E\left[\sum_{t=1}^T\q_t^{\top}(\R_t)^2\right]\nonumber\\
&\leq  \frac{\ln K}{\eta} +\eta\E\left[\sum_{t=1}^T2\q_t^{\top}(\widehat\r_t)^2 + 2\lambda_t^2\q_t^{\top}(\widehat\c_t)^2\right]
\leq  \frac{\ln K}{\eta} +  \frac{2\eta K T}{1-\gamma} + \frac{2\eta K}{1-\gamma}\sum_{t=1}^T\lambda_t^2,
\end{align}
where the third inequality follows from  the following inequality 
 \begin{align}\label{eqn:eqb}
 \E[\q_t^{\top}(\widehat\c_t)^2]&=\E\left[q^t_{i_t}\left(\frac{b^t_{i_t}}{p^t_{i_t}}\right)^2\right]\leq \frac{1}{1-\gamma}\E\left[p^t_{i_t}\left(\frac{c^t_{i_t}}{p^t_{i_t}}\right)^2\right]\nonumber\\
 &= \frac{1}{1-\gamma}\E\left[\frac{(c^t_{i_t})^2}{p^t_{i_t}}\right] = \frac{1}{1-\gamma}\E\left[\sum_{i=1}^K(c^t_i)^2\right]\leq \frac{K}{1-\gamma},
 \end{align}
 and the same inequality  holds for $\E[\q_t^{\top}(\widehat\r_t)^2]$. Next, we let $g_t(\lambda)=\frac{\delta}{2}\lambda^2 + \lambda(\q_t^{\top}\widehat\c_t - c_0)$. By following the similar analysis in the proof of Lemma~\ref{lemma:dual}, we have  
\begin{align*}
g_t(\lambda_t) - g_t(\lambda) 
&\leq \frac{1}{2\eta}\left(|\lambda-\lambda_t|^2-|\lambda-\lambda_{t+1}|^2\right) +  \frac{\eta}{2}|\nabla g_t(\lambda_t)|^2\\
&\leq  \frac{1}{2\eta}\left(|\lambda-\lambda_t|^2-|\lambda-\lambda_{t+1}|^2\right)  + \eta(\q_t^{\top}\widehat\c_t)^2 + \eta.
\end{align*}
Taking summation and expectation, we have
\begin{align}\label{eqn:dual2}
\E\left[\sum_{t=1}^Tg_t(\lambda_t) -  g_t(\lambda)\right]&\leq \frac{\lambda^2}{2\eta} + \eta \E\left[\sum_t \q_t^{\top} (\widehat \c_t)^2 \right] + \eta T.
\leq \frac{\lambda^2}{2\eta} +  \frac{\eta K T}{1-\gamma} + \eta T.
\end{align}
Combining equations~(\ref{eqn:dual2}) and~(\ref{eqn:primal2}) gives
\begin{align*}
&\E\left[\sum_{t=1}^T \p^{\top} \r_t  - \q_t^{\top} \r_t \right]+\E\left[\sum_{t=1}^T \lambda (c_0-\q_t^{\top}\c) - \left(\frac{\delta T}{2}+\frac{1}{2\eta}\right)\lambda^2\right]\\
&\leq  \frac{\ln K}{\eta}+\frac{4\eta K T}{1-\gamma} + \left(\frac{2\eta K}{1-\gamma} - \frac{\gamma}{2}\right) \sum_{t=1}^T\lambda_t^2 +\E\left[\sum_{t=1}^T\lambda_t(c_0-\p^{\top}\c)\right].
\end{align*}
Noting that $(1-\gamma)\q_t\leq \p_t$, so we get
\begin{align*}
&\E\left[\sum_{t=1}^T(1-\gamma) \p^{\top} \r_t  - \p_t^{\top} \r_t \right]+\E\left[\sum_{t=1}^T \lambda ((1-\gamma)c_0-\p_t^{\top}\c) - \left(\frac{\delta T}{2}+\frac{1}{2\eta}\right)\lambda^2\right]\\
&\leq  \frac{\ln K}{\eta}+ 4\eta K T+ \left(2\eta K -(1-\gamma)\frac{\delta}{2}\right) \sum_{t=1}^T\lambda_t^2 +\E\left[\sum_{t=1}^T\lambda_t(c_0-\p^{\top}\c)\right].
\end{align*}
Let $c_0\geq\p^{\top}\c$, $2\eta K\leq (1-\gamma)\frac{\delta}{2}$.  By taking maximization over $\lambda$, we have
\begin{align*}
&\E\left[\max_{\p^{\top}\c\geq c_0}\sum_{t=1}^T \p^{\top} \r_t  - \p_t^{\top} \r_t \right]+\E\left[\frac{\left[\sum_{t=1}^T((1-\gamma)c_0-\p_t^{\top}\c)\right]_+^2}{2(\delta T+1/\eta)}\right] \\ 
&\leq  \frac{\ln K}{\eta}+ 4\eta K T + \gamma T= \frac{ K(\delta+1)\ln K}{\delta\gamma} + 4\frac{\gamma \delta}{\delta+1} T+ \gamma T\\
&\leq \frac{ K(\delta+1)\ln K}{\delta \gamma} + \frac{ 5\delta+1}{\delta+1}\gamma T\leq \sqrt{\frac{(5\delta +1)K\ln K}{\delta}T}.
\end{align*}
Then we obtain
\begin{align*}
\max_{\p^{\top}\c\geq c_0}\sum_{t=1}^T \p^{\top} \r_t  - \E \left[\sum_{t=1}^{T}{\p_t^{\top} \r_t} \right] & \leq \sqrt{\frac{(5\delta +1)K\ln K}{\delta}T}\\
\E\left[\sum_{t=1}^T(c_0-\p_t^{\top}\c)\right]_+&\leq \sqrt{\left(T+ \sqrt{\frac{(5\delta +1)K\ln K}{\delta}T}\right)2(\delta T + 1/\eta)+\gamma T}.
\end{align*}
Let $\gamma= O(T^{-1/2}), \delta = O(T^{-1/2})$, then we get $O(T^{3/4})$ regret  and $O(T^{3/4})$ constraint bounds as claimed. 
\end{proof}
As our previous results,  we  present an algorithm with a high probability bound on the regret and the violation of the constraint. For ease of exposition, we introduce $\overline \c_t=\frac{1}{t}\sum_{s=1}^t \c_s$ and $ \widetilde\c_t=\frac{1}{t}\sum_{s=1}^t\widehat\c_s$. We modify BanditLEWA algorithm so that it uses more accurate estimations rather than using correct expectation in updating the primal and dual variables. To this end, we use upper confidence bound for rewards as Exp3.P algorithm~\cite{DBLP:journals/ml/AuerCF02} and for constraint vector $\c$. 
The following theorem states the regret bound and the violation of constraints in long term for the high probability BanditLEWA. 
\begin{figure}[t]
  \begin{center}
\begin{algorithm}
{\bf \large High Probability BanditLEWA($\eta$, $\gamma$, $\delta$, {\textnormal {and}} $\epsilon$)} \\
initialize: \+  $\w_1 =\exp\left(\eta \alpha\sqrt{KT} \right) \boldsymbol{1}$, and $\lambda_1 = \mathbf{0}$ , where $\alpha=2\sqrt{\ln(4KT/\epsilon)}$\- \\
{\bf iterate} $t=1,2,\ldots, T$  \+ \\
Set $\q_t=\w_t/ \sum_j w^t_j$ \\
Set $\p_t=(1-\gamma)\q_t + \gamma/K$ \\
Draw action $i_t$ randomly accordingly to the probabilities $\p_t$ \\
Receive reward $r^t_{i_t}$ and a realization of constraint  $c^t_{i_t}$  for action $i_t$\\
Update $\w_{t+1}$ by \\
\hspace{2cm}$ \displaystyle w^{t+1}_i = \exp\left(\eta\left[\left(\widehat r^t_i + \frac{\alpha}{p^t_i\sqrt{KT}}\right)+ \lambda_t\left(\widetilde c^t_i + \frac{2K}{\gamma}\frac{\alpha_1}{\sqrt{t}}\right)\right]\right) $ \\

Update $\lambda_{t+1} = [(1-\delta\eta)\lambda_{t} - \eta(\x_t^{\top}\widehat\c_t+\alpha_t -c_0)]_+$ \- \\
 {\bf end iterate}
\end{algorithm}
\end{center}
\label{alg:lewa4}
\end{figure}
\begin{theorem}
\label{thm:high:bandit:lewa}
Let $\alpha_t=\sqrt{(1/2)\ln(6KT/\epsilon)}/\sqrt{t}$, $\gamma = O(T^{-1/2}), \displaystyle \eta=\frac{\gamma}{\beta K}\frac{\delta}{\delta+ 1}$, and $\alpha= 2\sqrt{\ln(4KT/\epsilon)}$, where $\beta=\max\{3, 1+2\alpha_1\}$,  by running High Probability BanditLEWA, we have with probability $1-\epsilon$
\begin{align*}
\max_{\p^{\top}\c\geq c_0}\p^{\top}\sum_{t=1}^T\r_t- \sum_{t=1}^Tr^t_{i_t}& \leq O(T^{3/4}/\sqrt{\delta}) \;\;\text{and}\;\; \left[\sum_{t=1}^T(c_0-\p_t^{\top}\c)\right]_+\leq O(\sqrt{\delta}T).
\end{align*}
\end{theorem}
The proof  is deferred to~\ref{app-proof-thm7}. From this theorem, when $\delta =O(T^{-1/4})$, the regret and the  violation bounds are $O(T^{7/8})$ and $O(T^{7/8})$, respectively.

\section{Conclusions and Future Works}
In this paper we proposed an efficient algorithm for regret minimization under stochastic constraints. The proposed algorithm, namely LEWA,  is a primal-dual variant of the exponentially weighted average algorithm and relies on the theory of Lagrangian theory in constrained optimization. We establish expected and high probability bounds on the regret and the long term violation of the constraint in full information and bandit settings using novel theoretical analysis. In particular, in full information setting, LEWA algorithms attains optimal $\tilde{O}(\sqrt{T})$ regret bound and $O(T^{3/4})$ bound on the violation of the constraints in expectation, and with a simple trick in high probability. 

The present work leaves open a number of interesting directions for future work. In particular, extending the framework to handle multi-criteria online decision making is left to future work. Turning the proposed algorithm to the one which exactly satisfies the constraint in the long run is also an interesting problem. Finally, it would be interesting to see if it is possible to improve the bound obtained for the violation of the constraint.  


\section{References}
\bibliographystyle{elsarticle-num}
\bibliography{lewa}

\appendix

\section{Variation Bound for Violation of the Constraint}
\label{variation-proof}
Previously, when deriving the bound for the violation of the constraint, we simply bound the regret  as $\sum_{t=1}^T(\p^{\top}\r_t-\p_t^{\top}\r_t)\geq -T$. Since this simple lower bounding seems to be weak in general,  we present variation based bound for the violation of the constraint which results in  significantly improved bounds when  the variation of the consecutive reward vectors is small. For example, when the rewards vectors are correlated, the variation will be smaller than $T$.  We note that bounding the regret in terms of the variation of the reward vectors has been investigated in few recent  works~\cite{DBLP:conf/colt/HazanK08,gradual-var} and online learning algorithms with improved regret bound have been developed. To this end, let $\widehat\r_T=\frac{1}{T}\sum_{t=1}^T\r_t$ denote the mean of reward vectors $\r_t,t=1,\cdots, T$, and define the variation in the reward vectors as 
\begin{align*}
\text{ Variation}_T = \sum_{t=1}^T \|\r_t-\widehat \r_T\|_{\infty}.
\end{align*}
Then we have
\begin{align*}
\sum_{t=1}^T (\p_t^{\top}\r_t-\p^{\top}\r_t) &= \sum_{t=1}^T \p_t^{\top}(\r_t-\widehat\r_T) + (\p_t^{\top}\widehat\r_T - \p^{\top}\widehat\r_T) + \p^{\top}(\widehat\r_T-\r_t)\\
&\leq 2\text{ Variation}_T + \sum_{t=1}^T\p_t^{\top}\widehat\r_T  - \p^{\top}\widehat\r_T \\
&\leq 2\text{ Variation}_T + T(\widehat\p_T^{\top} \widehat\r_T - \p^{\top}\widehat \r_T),
\end{align*}
where $\widehat\p_T=\frac{1}{T}\sum_{t=1}^T\p_t$. The following lemma bounds the second term in above inequality. 

\begin{lemma}
\label{lemma:8}
Let $\p=\arg\max_{\x\in\Delta, \x^{\top}\c\geq c_0}\x^{\top}\widehat\r_T$,  then 
\begin{align*}
\widehat\p_T^{\top}\widehat\r_T - \p^{\top}\widehat\r_T \leq \frac{C}{T}\left[\sum_{t=1}^T(c_0- \p_t^{\top}\c)\right]_+
\end{align*}
where $C$ is some constant and $\Delta  =\{\alpha \in \R_+^K: \sum_{i=1}^K \alpha_i = 1 \}$ is the simplex. 
\end{lemma}
\begin{proof}
Let $h(\gamma)$ denote
\begin{align*}
h(\gamma) = \max_{\x\in\Delta} \x^{\top}\widehat\r_T, \quad \text{s.t. }  c_0-\x^{\top}\c\leq \gamma. 
\end{align*}
We assume $c_0-\widehat\p_T^{\top}\c\geq 0$, otherwise the bound is trivial. 
Then 
\begin{align*}
\widehat\p_T^{\top}\widehat\r_T - \p^{\top}\widehat\r_T\leq h(c_0- \widehat\x_T^{\top}\c)- h(0).
\end{align*}
Introducing Lagrangian multiplier, 
\begin{align*}
h(\gamma) &= \min_{\mu\geq 0} \max_{\x\in\Delta}\x^{\top}\widehat\r_T + \mu(\gamma- c_0 + \x^{\top}\c) \\ &=\min_{\mu\geq 0} \max_{\x\in\Delta} \x^{\top}(\widehat\r_T+\mu \c) -\mu c_0 + \mu\gamma \\
& =\min_{\mu\geq 0} g(\mu) + \mu\gamma.
\end{align*}
Since $g(\mu)+ \mu\gamma$ is concave in $\gamma$, therefore $h(\gamma)$ is also a concave function in $\gamma$. Then we have
\begin{align*}
\widehat\p_T^{\top}\widehat\r_T - \p^{\top}\widehat\r_T & \leq h(c_0- \widehat\p_T^{\top}\c)- h(0)\\ 
& \leq h'(0) (c_0- \widehat\p_T^{\top}\c) \\& \leq h'(0)\frac{1}{T}\left[\sum_{t=1}^T(c_0 -\p_t^{\top}\c)\right]_+,
\end{align*}
where the last inequality follows that fact that   $h(\gamma)$ is a monotonically increasing function, i.e.,  $h'(0)\geq 0$. 
\end{proof}
From the proof, the condition in Lemma \ref{lemma:8} holds if $h'(0)$ exists. In order to show $h'(0)$ exists, we need to  show that the linear system,
\begin{align}\label{eqn:l1}
\max_{\x} \x^{\top}\r, \text{ s.t. } c_0-\c^{\top}\x\leq 0 , \x^{\top}\mathbf 1=1, \x\geq 0
\end{align}
and its dual satisfy the regular condition. In order to represent the above linear programming problem in a  standard form, we let $\mathbf A=(-\c, \mathbf 1, -\mathbf 1)^{\top}$ and $\mathbf u=(-c_0, 1, -1)^{\top}$, and rewrite the  linear system in~(\ref{eqn:l1})  as 
\begin{align*}
\max_{\x} \quad&\x^{\top}\r\\
\text{s.t.}  \quad & \mathbf A\x\leq \mathbf u, \x\geq 0,
\end{align*}
and its dual problem as 
\begin{align*}
\min_{\mathbf y}\quad & \mathbf y^{\top}\mathbf u\\
\text{s.t.}\quad & \mathbf y\geq 0, \mathbf A^{\top}\mathbf y\geq \mathbf r.
\end{align*}
To show the system satisfy the regular condition, we need to show that 
\begin{align}
\x\succeq 0, \mathbf A\x\leq 0\Rightarrow \x^{\top}\r<0\label{eqn:c1}\\
\mathbf y\succeq 0, \mathbf A^{\top}\mathbf y\geq 0\Rightarrow \mathbf y^{\top}\mathbf u>0\label{eqn:c2}
\end{align}
where $\succeq 0$ denotes at least one element is positive, which is also termed semi-positive. To prove ~(\ref{eqn:c1}), note that there does not exists any semipositive vector $\x$  such that $\x^{\top}\mathbf 1=0$.  Therefore the primal system  satisfy the regular condition vacuously. Although the primal system does not satisfy the regular condition, the dual system still satisfy the regular condition as long as $c_0<c_{\max}$. The gradient $h'(0)$ is actually the Lagrangian variable when $\gamma=0$. The following lemma verifies the existence of $h'(0)$. 
\begin{lemma}
\label{lemma:41}
Let $\x\geq 0, \x^{\top}\mathbf 1=1, \x^{\top}\c\geq c_0$ be strictly feasible or $c_0<\max_k c_k$, then their exists bounded gradient $h'(0)$. 
\end{lemma}
Following Lemma \ref{lemma:8}, we have
\begin{align*}
\sum_{t=1}^T (\p_t^{\top}\r_t-\p^{\top}\r_t) \leq 2\text{ Variation}_T +  C\left[\sum_{t=1}^T(c_0- \p_t^{\top}\c)\right]_+
\end{align*}
 Then we have
 \begin{align*}
 \left[\sum_{t=1}^T(c_0- \p_t^{\top}\c)\right]_+^2\leq O\left(\sqrt{T}\right)\left(2\text{ Variation}_T+C\left[\sum_{t=1}^T(c_0- \p_t^{\top}\c)\right]_+ + O\left(\sqrt{T}\right)\right) +  O(\sqrt{T})
 \end{align*}
 Then we get 
 \begin{align*}
 \left[\sum_{t=1}^T(c_0- \p_t^{\top}\c)\right]_+\leq  O(\sqrt{T})+ O(T^{1/4}\sqrt{\text{ Variation}_T})
 \end{align*}

\section{Proof of Theorem~\ref{thm:high:bandit:lewa}}
\label{app-proof-thm7}
Similar to the analysis for Exp3.P algorithm in~\cite{DBLP:journals/ml/AuerCF02}, we have have the following two upper confidence bounds, 
\begin{align}
\sum_{t=1}^T\widehat r_i^t + \alpha \sigma_i^t\geq \sum_{t=1}^T r_i^t, \forall i\\
\sum_{t=1}^T\left(\widetilde c_i^t + \frac{2K}{\gamma}\frac{\alpha_1}{\sqrt{t}}\right)\geq \sum_{t=1}^T \overline c_i^t, \forall i
\end{align}
where $\sigma^t_i= \sqrt{KT} + \frac{1}{KT}\sum_{s=1}^t1/p^s_i$.  Following the same line of proof as  in~\cite{DBLP:journals/ml/AuerCF02},  we have
\begin{align*}
\sum_{t=1}^T\ln \frac{W_{t+1}}{W_t}&\leq \eta\sum_{t=1}\q_t^{\top}(\widehat \r_t+ \lambda_t\widetilde \c_t) + \frac{\alpha\eta}{1-\gamma}\sqrt{kT}+ \frac{4\alpha_1\eta K}{\gamma\delta}\sqrt{T} \\
&\hspace{0.5cm}+ \frac{4\eta^2}{1-\gamma}\sum_{t=1}^T \sum_{i}\widehat r^t_i + \frac{4\eta^2}{1-\gamma}\sum_{t=1}^T\lambda_t^2\widetilde\c_t^{\top}\mathbf 1 + \frac{4\alpha^2 \eta^2}{\gamma(1-\gamma)} + \frac{16\alpha_1^2\eta^2K^2}{\gamma^2\delta^2}(1+\ln(T))
\end{align*}
and
\begin{align*}
\sum_{t=1}^T\ln \frac{W_{t+1}}{W_t}&\geq\eta\p^{\top}\sum_{t=1}^T(\r_t + \lambda_t\overline\c_t ) -\ln K \\
&\hspace{0.5cm}+\eta\left( \sum_{t=1}^T(\widehat r^t_i + \alpha\sigma^t_i + \lambda_t\widetilde c^t_i + \lambda_t\frac{2K}{\gamma}\alpha_t) - \p^{\top}\sum_{t=1}^T(\r_t +  \lambda_t\overline\c_t) \right). 
\end{align*}
Then we have
\begin{align*}
&\p^{\top}\sum_{t=1}^T(\r_t + \lambda_t\overline\c_t ) +\left( \sum_{t=1}^T(\widehat r^t_i + \alpha\sigma^t_i + \lambda_t\widetilde c^t_i + \lambda_t\frac{2k}{\gamma}\alpha_t) - \p^{\top}\sum_{t=1}^T( \r_t + \lambda_t\overline\c_t ) \right) - \q_t^{\top}(\widehat \r_t+ \lambda_t\widetilde \c_t) \\
&\leq  \frac{\alpha}{1-\gamma}\sqrt{KT}+\frac{4\alpha_1K}{\gamma\delta}\sqrt{T} + \frac{4\eta}{1-\gamma}\sum_{t=1}^T \sum_{i}\widehat r^t_i + \frac{4\eta}{1-\gamma}\sum_{t=1}^T\lambda_t^2\widetilde\c_t^{\top}\mathbf 1 \\ &\hspace{0.5cm}+\frac{4\alpha^2 \eta}{\gamma(1-\gamma)} + \frac{16\alpha_1^2\eta K}{\gamma^2\delta^2}(1+
\ln T)+ \frac{\ln K}{\eta}
\end{align*}
On the other side, let $g_t(\lambda)= \frac{\delta}{2}\lambda^2 + \lambda(\q_t^{\top}\widehat\c_t + \alpha_t -c_0)$, with probability $1-\epsilon/4$, we have
\begin{align*}
g_t(\lambda_t) - g_t(\lambda)&\leq \frac{1}{2\eta}\left(|\lambda-\lambda_t|^2-|\lambda-\lambda_{t+1}|^2\right) +  \frac{\eta}{2}|\nabla_\lambda g_t( \lambda_t)|^2\\
&\leq  \frac{1}{2\eta}\left(|\lambda-\lambda_t|^2-|\lambda-\lambda_{t+1}|^2\right)  + \eta/2(\x_t^{\top}\widehat\c_t-c_0+\alpha_t+\delta\lambda_t)^2 \\
&\leq  \frac{1}{2\eta}\left(|\lambda-\lambda_t|^2-|\lambda-\lambda_{t+1}|^2\right)  + \eta(\x_t^{\top}\widehat\c_t)^2 + \eta C\\
&\leq  \frac{1}{2\eta}\left(|\lambda-\lambda_t|^2-|\lambda-\lambda_{t+1}|^2\right)  + \eta\x_t^{\top}(\widehat\c_t)^2 + \eta C\\
&\leq  \frac{1}{2\eta}\left(|\lambda-\lambda_t|^2-|\lambda-\lambda_{t+1}|^2\right)  + \frac{\eta}{1-\gamma}\mathbf 1^{\top}\widehat\c_t + \eta C\\
&\leq  \frac{1}{2\eta}\left(|\lambda-\lambda_t|^2-|\lambda-\lambda_{t+1}|^2\right)  + \frac{\eta}{1-\gamma}(\mathbf 1^{\top}\c+ \frac{K}{\gamma}\alpha_t)+ \eta C\\
\end{align*}
where $C=(1+\alpha_1)^2$, $\displaystyle \alpha_t=\alpha_1/\sqrt{t}$.  Taking summation over $t=1,\cdots, T$ of above inequalities, we have
\begin{align*}
&\sum_{t=1}^T\frac{\delta}{2}\lambda_t^2- \lambda_t(c_0-\alpha_t-\q_t^{\top}\widehat \c_t)  + \lambda(c_0-\alpha_t-\q_t^{\top}\widehat \c_t) - \frac{\delta}{2}\lambda^2 \\
& \leq \frac{\lambda^2}{2\eta} +\sum_{t=1}^T\frac{\eta}{1-\gamma}(\mathbf 1^{\top}\c+ \frac{K}{\gamma}\alpha_t)+ \eta CT
\end{align*}
Combing the primal inequality and the dual inequality, we have
\begin{align*}
&\left( \sum_{t=1}^T(\widehat r^t_i + \alpha\sigma^t_i + \lambda_t\widetilde c^t_i + \lambda_t\frac{2K}{\gamma}\alpha_t) - \p^{\top}\sum_{t=1}^T(\r_t + \lambda_t\overline\c_t ) \right) \\
&+ \sum_{t=1}^T\p^{\top}(\r_t+\lambda_t\overline \c_t) - \q_t^{\top}\widehat\r_t - \lambda_t(c_0-\alpha_t)+ \frac{\delta}{2}\lambda_t^2 +  \lambda(c_0-\alpha_t-\q_t^{\top}\widetilde \c_t) - \frac{\delta}{2}\lambda^2\\
&\leq \frac{\lambda^2}{2\eta} +\sum_{t=1}^T\frac{\eta}{1-\gamma}(\mathbf 1^{\top}\c+ \frac{K}{\gamma}\alpha_t)+ \eta CT+   \frac{\alpha}{1-\gamma}\sqrt{KT}+\frac{4\alpha_1K}{\gamma\delta}\sqrt{T} + \frac{4\eta}{1-\gamma}\sum_{t=1}^T \sum_{i}\widehat r^t_i \\
&+ \frac{4\eta}{1-\gamma}\sum_{t=1}^T\lambda_t^2\widetilde\c_t^{\top}\mathbf 1 + \frac{4\alpha^2 \eta}{\gamma(1-\gamma)} + \frac{16\alpha_1^2\eta K}{\gamma^2\delta^2}(1+\ln T)+ \frac{\ln K}{\eta}.
\end{align*}
Then with probability $1-\epsilon$,  we have the following inequality:
\begin{align*}
&\left( \sum_{t=1}^T(\widehat r^t_i + \alpha\sigma^t_i + \lambda_t\widetilde c^t_i + \lambda_t\frac{2K}{\gamma}\alpha_t) - \p^{\top}\sum_{t=1}^T(\r_t + \lambda_t\overline\c_t) \right) \\
&+ \sum_{t=1}^T\p^{\top}(\r_t+\lambda_t\overline \c_t) - \q_t^{\top}\widehat\r_t - \lambda_t(c_0-\alpha_t) + \frac{\delta}{2}\lambda_t^2 +  \lambda(c_0-\alpha_t-\q_t^{\top}\widetilde \c_t) - \frac{\delta}{2}\lambda^2\\
&\leq \frac{\lambda^2}{2\eta} +\sum_{t=1}^T\frac{\eta}{1-\gamma}(\mathbf 1^{\top}\c+ \frac{K}{\gamma}\alpha_t)+ \eta CT+   \frac{\alpha}{1-\gamma}\sqrt{KT}+\frac{4\alpha_1 k}{\gamma\delta}\sqrt{T} + \frac{4\eta}{1-\gamma}\sum_{t=1}^T \sum_{i}\widehat r^t_i \\
&\hspace{0.5cm}+ \frac{4\eta}{1-\gamma}\sum_{t=1}^T\lambda_t^2\widetilde\c_t^{\top}\mathbf 1 + \frac{4\alpha^2 \eta}{\gamma(1-\gamma)} + \frac{16\alpha_1^2\eta K}{\gamma^2\delta^2}(1+\ln T)+ \frac{\ln K}{\eta}.
\end{align*}
Let $\widehat U_T=\max_{i}\sum_{t=1}^T (\widehat r_i^t + \alpha\sigma_i^t + \lambda_t(\widetilde c_i^t +\frac{2K}{\gamma}\alpha_t))$, $\eta= \frac{\gamma}{\beta K }\frac{\delta}{\delta+1}, \gamma\leq (\beta)/(4+\beta)$, then we have
\begin{align*}
&\left(1- \frac{4\gamma}{\beta(1-\gamma)}\right)\widehat U_T -\p^{\top}\sum_{t=1}^T(\r_t + \lambda_t\overline\c_t )\\
&+ \sum_{t=1}^T\p^{\top}(\r_t+\lambda_t\overline \c_t) - \q_t^{\top}\widehat\r_t - \lambda_t(c_0-\alpha_t) + \sum_{t=1}^T \lambda(c_0-\alpha_t-\q_t^{\top}\widetilde \c_t) - \left(\frac{\delta T}{2}+\frac{1}{2\eta} \right)\lambda^2\\
&\leq \sum_{t=1}^T \frac{\eta}{1-\gamma}(\mathbf 1^{\top}\c+ \frac{K}{\gamma}\alpha_t)+ \eta CT+   \frac{\alpha}{1-\gamma}\sqrt{KT} \\
& \hspace{0.5cm}+
\frac{4\alpha_1 K}{\gamma\delta}\sqrt{T}  + \frac{4\alpha^2 \eta}{\gamma(1-\gamma)} + \frac{16\alpha_1^2\eta K}{\gamma^2\delta^2}(1+\ln T)+ \frac{\ln K}{\eta}.
\end{align*}
Since $\widehat U_T\geq\max_{i}\sum_{t=1}^Tr_i^t +\lambda_t \overline c_i^t$,  and $\p^{\top}\sum_{t=1}^T(\r_t + \lambda_t\overline\c_t )\leq \max_{i}\sum_{t=1}^Tr_i^t +\lambda_t \overline c_i^t$, then we have with probability $1-\epsilon$, 
\begin{align*}
&\sum_{t=1}^T\p^{\top}\r_t- \q_t^{\top}\widehat\r_t - \lambda_t(c_0-\alpha_t-\p^{\top}\overline\c_t) + \sum_{t=1}^T \lambda(c_0-\alpha_t-\q_t^{\top}\widetilde \c_t) - \left(\frac{\delta T}{2}+\frac{1}{2\eta} \right)\lambda^2\\
&\leq \sum_{t=1}^T\frac{\eta}{1-\gamma}(\mathbf 1^{\top}\c+ \frac{K}{\gamma}\alpha_t)+ \eta CT+   \frac{\alpha}{1-\gamma}\sqrt{KT}+\frac{4\alpha_1 K}{\gamma\delta}\sqrt{T}  + \frac{4\alpha^2 \eta}{\gamma(1-\gamma)} \\
& \hspace{0.5cm}+ \frac{16\alpha_1^2\eta K}{\gamma^2\delta^2}(1+\ln T)+ \frac{\ln K}{\eta} +  \frac{4\gamma}{\beta(1-\gamma)}\max_i\left(\sum_{t=1}^Tr_i^t +\lambda_t \overline c_i^t\right)\\
&\leq \frac{\alpha_1\sqrt{T}}{1-\gamma} +\frac{\gamma T}{\beta(1-\gamma)}+ \frac{C\gamma T }{\beta}+   \frac{\alpha}{1-\gamma}\sqrt{KT}+\frac{4\alpha_1K}{\gamma\delta}\sqrt{T}\\
& + \frac{4\alpha^2}{\beta(1-\gamma)K} + \frac{16\alpha_1^2}{\beta\gamma\delta^2}(1+\ln T)+ \frac{\beta(K\ln K)}{\gamma}\frac{\delta+1}{\delta} + \frac{4\gamma T}{\beta(1-\gamma)}\frac{\delta+1}{\delta}.
\end{align*}
Then 
\begin{align*}
&\sum_{t=1}^T\p^{\top}\r_t- \p_t^{\top}\widehat\r_t  +\frac{\left[\sum_{t=1}^T((1-\gamma)(c_0-\alpha_t)-\p_t^{\top}\widetilde\c_t)\right]_+^2}{2(\delta T+1/\eta)} \\
& \leq \alpha_1\sqrt{T}+ \frac{C_1\gamma T }{\beta}+   \alpha\sqrt{KT}+\frac{4\alpha_1 k}{\gamma\delta}\sqrt{T}  + \frac{4\alpha^2}{\beta K}+\frac{16\alpha_1^2}{\beta\gamma\delta^2}(1+\ln T) \\
&\hspace{0.5cm}+\frac{\beta(K\ln K)}{\gamma}\frac{\delta+1}{\delta} + 4\gamma T\frac{\delta+1}{\beta\delta}. 
\end{align*}
Let $\gamma= O(T^{-1/4}), \eta= O(T^{-1/4})$, then we obtain
\begin{align*}
\max_{\p^{\top}\c\geq c_0}\p^{\top}\sum_{t=1}^T\r_t- \sum_{t=1}^Tr^t_{i_t}& \leq O(T^{3/4}/\sqrt{\delta})\;\;\;\text{and}\\
\left[\sum_{t=1}^T(c_0-\p_t^{\top}\c)\right]_+&\leq O(\sqrt{\delta}T),
\end{align*}
when $\delta =O(T^{-1/4})$, the regret bound is $O(T^{7/8})$, the worse case constraint bound is $O(T^{7/8})$.

\end{document}